\documentclass[12pt]{article}
\usepackage{amsmath,amssymb,amsfonts,bbm,bm,cases,verbatim,multirow,color,xcolor}
\usepackage[mathscr]{euscript}
\usepackage{epic,eepic,psfrag,epsfig}
\usepackage{graphicx}
\usepackage{amsthm,breakcites}
\usepackage{amscd}
\usepackage{epsfig}
\usepackage{fullpage}

\usepackage{bbm}

\usepackage[linesnumbered, ruled,vlined]{algorithm2e}
\usepackage{subfig}

\DeclareMathOperator{\E}{\mathbb{E}}
\DeclareMathOperator{\N}{\mathbb{N}}

\DeclareMathOperator{\Prob}{\mathbb{P}}
\DeclareMathOperator*{\argmax}{argmax}
\DeclareMathOperator{\1}{\mathbbm{1}}
\DeclareMathOperator{\KL}{KL}

\newcommand{\klucb}{\mathrm{U}}

\newtheorem{defn}{Definition}[section]

\newtheorem{thm}[defn]{Theorem}
\newtheorem{prop}[defn]{Proposition}
\newtheorem{cor}[defn]{Corollary}
\newtheorem{lemma}[defn]{Lemma}

\usepackage{parskip}

\usepackage{xcolor}
\usepackage[draft,inline,nomargin,index]{fixme}
\fxsetup{theme=color,mode=multiuser}
\FXRegisterAuthor{cn}{acn}{\color{blue} CN}
\FXRegisterAuthor{hr}{ahr}{\color{red} HR}
\FXRegisterAuthor{ag}{aag}{\color{green} AG}
\newcommand\nth{\textsuperscript{th}\xspace} 
\newcommand{\diam}{\mathrm{diam}}

\title{Asymptotic Optimality for Decentralised Bandits}
\author{Conor Newton, Ayalvadi Ganesh, Henry Reeve}
\date{School of Mathematics, University of Bristol}

\begin{document}
\maketitle

\begin{abstract}
We consider a large number of agents collaborating on a multi-armed bandit problem with a large number of arms. The goal is to minimise the regret of each agent in a communication-constrained setting. We present a decentralised algorithm which builds upon and improves the \emph{Gossip-Insert-Eliminate} method of Chawla \emph{et al.}  \cite{chawla2020gossiping}. We provide a theoretical analysis of the regret incurred which shows that our algorithm is asymptotically optimal. In fact, our regret guarantee matches the asymptotically optimal rate achievable in the full communication setting. Finally, we present empirical results which support our conclusions.
\end{abstract}

\section{Introduction}

The classical stochastic multi-armed bandit problem is specified by a collection of probability distributions $\{P_{k}\}_{k =1}^K$, commonly referred to as arms. Here, there is a single agent which plays an arm $I_t$ taking values in $[K]:=\{ 1,\ldots,K \}$ at each time step $t\in [T]$ and receives an associated reward $X_t \sim P_{I_t}$. The agent's goal is to minimise the expected regret $\E[{R}_{T}]= T \mu_\star - \sum_{t=1}^{T}\mathbb{E}[X_{t}]$, where $\mu_k$ is the expectation of a random variable with distribution $P_k$, and $\star:=\argmax_{k \in [k]} \mu_k$ is the largest mean of the arms. The agent's decisions must be made using only the knowledge acquired from previous actions and observed rewards.

Motivated by applications in distributed computing, we consider a collection of agents collaborating on a multi-armed bandit problem  \cite{sankararaman2019social,chawla2020gossiping}. Agents may communicate with one another, and an agent's decision of which arm to play is made using information derived both from their own reward history, and from the sequence of messages received from other agents. However, communication between agents is tightly restricted as described in Section \ref{sec:SettingAndAlgo}. Specifically, time is divided into growing phases and each agent may receive only one message per phase. Furthermore, a message is limited to recommending the id of a single arm; no additional information may be exchanged. We show in Theorem \ref{thm:asymptoticRegretBoun} that, even with these restrictions on communication, it is possible to asymptotically match the optimal total regret achievable with unlimited communication.

There has recently been growing interest in multi-agent multi-armed bandits. A setting in which agents communicate with a central node is considered in \cite{kanade2012distributed}, while  \cite{szorenyi2013gossip,cesa2020cooperative,martinez2018decentralized,dubey2020cooperative} consider settings where agents can communicate \emph{rewards} (not just arm ids) with their neighbours. We follow the setting introduced in \cite{sankararaman2019social, chawla2020gossiping} 
where agents may only communicate arm ids. 
In recent work, \cite{agarwal2021multi} introduced a method for achieving nearly minimax optimal regret in this setting.

A central problem in the multi-armed bandit literature is the search for algorithms which perform optimally in the asymptotic regime of the time-horizon $T$ tending to infinity. Returning to the single-agent setting, Lai \& Robbins \cite{LAI19854} proved a fundamental lower bound on the regret incurred by any \emph{consistent} algorithm. Here, we say that an algorithm is \emph{consistent} if it achieves sub-polynomial regret for all possible values of $\{P_{k}\}_{k =1}^K$. (This precludes trivial algorithms like one which always selects a specific arm and has zero regret if that happens to be the best arm.) Lai \& Robbins \cite{LAI19854} showed that the regret of any consistent algorithm satisfies the following lower bound:
\begin{align}\label{eq:laiRobbinsLB}
\liminf_{T \to \infty}\frac{\sum_n\E[\mathcal{R}_T]}{\log (T)} 
    &\ge \sum_{i \neq \star} \frac{\mu_{\star} - \mu_i} {\KL(P_i, P_{\star})},
\end{align}
where $\KL$ denotes the Kullback-Leibler divergence. A significant breakthrough was achieved by \cite{garivier2011kl} and \cite{maillard2011finite} who demonstrated that this bound is attained by the KL-UCB algorithm in the Bernoulli setting.

In this work, we consider the question of asymptotic optimality in the decentralised multi-agent setting. Our contributions are as follows:
\begin{itemize}
    \item We present a decentralised algorithm which builds upon and improves the \emph{Gossip-Insert-Eliminate} method of Chawla \emph{et al.}  \cite{chawla2020gossiping}. This algorithm leverages two innovations which reduce the amount of superfluous exploration. Firstly, we include a more efficient elimination mechanism which reduces the number of arms considered by each agent at any given time. Secondly, in the spirit of \cite{garivier2011kl,maillard2011finite}, we use KL-type confidence intervals, rather than Hoeffding-type confidence intervals.  
    \item We provide a theoretical analysis of the expected regret of the algorithm we propose (Theorem \ref{thm:asymptoticRegretBoun}). We show that it is optimal in the asymptotic regime. In particular, the aggregate expected regret matches the lower bound implied by \eqref{eq:laiRobbinsLB}, showing that our algorithm performs at least as well as any multi-agent algorithm, even with access to unlimited communication resources, in the asymptotic regime.
    \item We present empirical results that demonstrate that our algorithm performs well in a wide variety of settings, with lower finite sample-regret than the base line of \cite{chawla2020gossiping} (Figures \ref{fig:alpha}, \ref{fig:delta2}). Interestingly, both modifications lead to a consistent improvement for a range of different values of the gap between best and second-best arm.
\end{itemize}

\section{Setting and algorithm}\label{sec:SettingAndAlgo}

We now present our problem setting and algorithm. Throughout $N$ will denote the number of agents, $T$ the number of time steps, and $K$ the number of arms. Let $X^n_{k,s}$ taking values in $\{ 0,1\}$ denote the reward that agent $n \in [N]$ receives by playing arm $k\in [K]$ for the $s\nth$ time. We assume that these are i.i.d. Bernoulli($\mu_k$) random variables. Let $\star \in \argmax \mu_k$ and let $\mu_{\star}:= \max_{k \in [K]}\mu_k$. We assume throughout that there is a unique best arm, so $\star$ is uniquely defined.

\newcommand{\numPlays}{V}

Communication between agents is constrained by a strictly increasing sequence $(A_j)_{j \in \N}$ and an $N\times N$ probability matrix $P$ as follows. The time horizon $[T]$ is partitioned into phases, with phase $j$ consisting of time steps $t$ for which $A_{j-1} <t \leq A_j$ where $A_0:=0$. Communication between agents only occurs once a phase, on time steps $A_j$. On this time step agents request a message from their neighbours. The communicating agent is selected randomly according to $P$, with $P(n,q)$ denoting the probability that agent $n$ will receive a message from agent $q$ at the end of each phase $j$. We let $Q\equiv Q^n_j \sim P(n,\cdot)$ be the random variable corresponding to the agent who sends a message to agent $n$ at the end of phase $j$. The message, from agent $Q^n_j$ to $n$, must take the form of an arm recommendation $O^j_n$, taking values in $[K]$.

Let $I^n_t$ denote the random variable, taking values in $[K]$, which specifies the index of the arm played by agent $n$ in round $t$. This must be a measurable function of an agent's previous reward history and the previous messages they have received. We let $\numPlays^n_k(t):= \sum_{s=1}^t \1\{I^n_s=k\}$ denote the number of times agent $n$ plays arm $k$ in the first $t$ rounds. Let $X^n(t):=X^n_{I^n_t,\numPlays_k^n(t)}$ denote the reward received by agent $n$ in round $t$.

The goal of each agent $n \in [N]$ is to minimise their expected regret,
\begin{align*}
\E[\mathcal{R}_T^n] :=T \cdot \mu_{\star} - \sum_{t \in [T]} \E[X^n(t)].
\end{align*}

Our algorithm (Algorithm \ref{alg:klalg}) is based on the Gossip-Insert-Eliminate algorithm of \cite{chawla2020gossiping}.
A key feature of this algorithm is that, during each phase $j$, each agent plays only a small subset of the $K$ arms which we call its ``active set''. This is made up of a ``sticky set'' of arms, which remains unchanged over time for each agent, and additional arms which evolve over time based on recommendations. We now describe how these sets are determined.

In our algorithm, we begin by partitioning $[K]$ into nearly equal-sized sets $\{{S}^n_{\circ}\}_{n \in [N]}$, so that for each agent $n \in [N]$, $S^n_\circ$ will act as the associated sticky set. 
The active sets are initialised to be the same as the sticky sets, but will grow over time due to recommendations and shrink due to eliminations of non-sticky arms. 
In each phase $j\in \N$, each agent $n \in [N]$ will only play arms from the active set $S^n_j$. 
For the first phase $j=1$ we initialise each $S^n_1= S^n_\circ$. In subsequent phases $j>1$ the active set $S^n_{j+1}$ consists of $S^n_\circ$, along with (potentially) additional arms. 

We assume that each agent $n$ is aware of  ${S}^n_{\circ}$, its own set of arms within the partition, \emph{a priori}. That is, ${S}^n_{\circ}$ may be taken as an input to our algorithm. Let $\hat{\mu}^n_{k,s}:=\frac{1}{s}\sum_{i=1}^s X^n_{k,i}$. Denote by $\hat{\mu}^n_k(t):=\hat{\mu}^n_{k,\numPlays^n_k(t)}$ the mean reward obtained by agent $n$ from arm $k$ in the first $t$ time steps.

We let $M_j^n$ denote the most played arm by agent $n$ in phase $j$ so $$M_j^n = \argmax_{k \in [K]} \{\numPlays_k^n(A_j)-\numPlays^n_k(A_{j-1})\}.$$ Following \cite{chawla2020gossiping}, when an agent $q \in [N]$ is asked for an arm recommendation at the end of phase $j$, its recommendation will be its most played arm for that phase. Hence, when $Q\equiv Q^n_j \sim P(n,\cdot)$ communicates with agent $n \in [N]$ at the end of phase $j$, the recommendation will be $O^j_n=M^{Q}_j$.

Our algorithm (Algorithm \ref{alg:klalg}) differs from that of \cite{chawla2020gossiping} in two important respects. 

Firstly, we use a more efficient elimination scheme. More precisely, in each phase $j+1$, the new active set $S^n_{j+1}$ will be constituted by the sticky set $S^\circ_n$, together with the agent's most played arm  $M^n_j$ during phase $j$, and the recommendation, $O^n_j$, it receives at the end of phase $j$. 
The intuition is that, eventually, the best arm will become known to all agents, and $M^n_j$ and $O^n_j$ will both be equal to $\star$; consequently, $S^n_j$ will be $S^n_{\circ}\cup \{ \star\}.$

Secondly, we use tighter KL based confidence intervals, following \cite{garivier2011kl}. To define our KL upper confidence bounds we first let $\KL : [0,1]^2 \to \mathbb{R} \cup \{\infty\}$ be the Kullback--Leibler divergence for two Bernoulli random variables and introduce a function $f_\alpha(t) = 1 + t^{\alpha} \log ^2(t)$ indexed by $\alpha$. The upper confidence bound for arm $k$ at agent $n$ at time $t$ is defined by
\begin{equation*}
    \klucb^{n}_{k,\alpha}(t - 1) := \max\left\{u \in [0, 1]: \KL(\hat\mu^{n}_{k}(t - 1), u) \le \frac{\log(f_\alpha(t))}{\numPlays_k^n(t - 1)}  \right\}\\
\end{equation*}
when $\numPlays_k^{n}(t - 1)>0$ and $\klucb^{n}_{k}(t - 1) := \infty$ otherwise. When $\alpha$ is clear from context we suppress it for notational convenience.

\begin{algorithm}
    \SetAlgoLined
    \DontPrintSemicolon
    $j \leftarrow  1$  and  $S_1^n \leftarrow S_\circ^n$ \\
    \For{$t \in \mathbb{N}$}{
        $I^{n}_{t} \leftarrow \argmax_{k \in {S}^{n}_j}\klucb^{n}_{k,\alpha}(t - 1)$\\
        \If{$t == A_{j}$}{
            $Q \leftarrow P(i, \cdot)$ and $O_j^n=M^{Q}_j$ \\
            $S_{j+1}^{n} \leftarrow {S}^{n}_{\circ} \cup \{O_j^n,M_j^n\}$\\
            $j \leftarrow j + 1$\\
        }
    }
    \caption{Asymptotically Optimal Gossiping Bandits (AOGB)}
    \label{alg:klalg}
\end{algorithm}

\section{Theoretical analysis and regret bound}

We now present our asymptotically optimal regret bound for Algorithm \ref{alg:klalg}.

\begin{thm}\label{thm:asymptoticRegretBoun} Suppose that $P$ has a strongly connected graph and there exist $C\geq 1$, $\theta>0$ such that $C^{-1} j^\theta \leq A_{j}-A_{j-1} \leq Cj^{\theta}$ for all $j \in \N$.  Suppose that all agents select arms with Algorithm \ref{alg:klalg} with $\alpha=1$. Then for each agent $n \in [N]$ we have the asymptotic bound
\begin{align*}
\limsup_{T \rightarrow \infty} \frac{\E[\mathcal{R}^n_T]}{\log T} \leq \sum_{ k \in S_{\circ}^n\backslash [\star]}\frac{\mu_{\star}-\mu_k}{\KL(\mu_k,\mu_{\star})}.
\end{align*}
\end{thm}

Note that by summing over the regrets of the different agents the regret bound above matches the lower bound for the full communication setting implied by \cite{LAI19854}. Indeed, let's consider the class of centralised algorithms $\mathcal{A}$ in which an arm $I^n_t$ in $[K]$ is selected for each agent $n\in [N]$ and each time step $t \in [T]$ based  on the combined reward history of all the agents up to time $t$. We let $\mathcal{A}_{\mathrm{const}}\subseteq \mathcal{A}$ denote the subset of those which are \emph{consistent} ie. achieve sub-polynomial total regret $\sum_n\E[\mathcal{R}^n_T]$ for any instance of the multi-armed bandit problem. It follows from the result of Lai and Robbins \eqref{eq:laiRobbinsLB} that for any algorithm in the class $\mathcal{A}_{\mathrm{const}}$,
\begin{align}
    \liminf_{T \to \infty}\frac{\sum_n\E[\mathcal{R}^n_T]}{\log (T)} = \liminf_{T \to \infty} \left(\frac{\log(NT)}{\log(T)} \cdot \frac{\sum_n\E[\mathcal{R}^n_T]}{\log (NT)}\right)
    \ge \sum_{i \neq \star} \frac{\mu_{\star} - \mu_i} {\KL(\mu_i, \mu_{\star})} \label{eq:ourLB}.
\end{align}
Now note that we can view the class $\mathcal{A}$ as the collection of all multi-agent algorithms, with or without communication constraints. In particular, the class of decentralised multi-agent with strong communication constraints we consider in this paper correspond to a computationally attractive subset of $\mathcal{A}$. Observe that by summing over $n \in [N]$ in the regret bound given in Theorem \ref{thm:asymptoticRegretBoun}, we see that total regret of the system for our algorithm matches the lower bound given by \eqref{eq:ourLB} for the full communication setting. This implies that our algorithm (with limited communication) performs just as well as any algorithm, even with access to unlimited  communication constraints, in the asymptotic regime.

Of course, our theoretical results only certify the performance of Algorithm \ref{alg:klalg} in the asymptotic regime. Nonetheless, in Section \ref{sec:numericalResults} we shall see that our algorithm also performs well empirically on a broad range of simulated data.

Before presenting the main proof of Theorem \ref{thm:asymptoticRegretBoun} we shall present a brief sketch. The  hinges upon a random time $\hat{\tau}$ which corresponds to the phase after which all of the active sets $S^n_j$ become fixed. After this random time all of the active sets become $S^n_\circ \cup \{\star\}$, which leads to an asymptotic regret bound for agent $n$ governed by the relationship between $\mu_k$ and $\mu_{\star}$ for $k \in [K]$. The crucial difficulty then is to bound $\E[A_{\hat\tau}]$, the expected time until the end of phase $\hat{\tau}$. To bound $\E[A_{\hat\tau}]$ we show that, provided the phase lengths $A_j-A_{j-1}$ are sufficiently large in relationship to the gap, the probability of a sub-optimal arm being the most played, and subsequently being recommended decays exponentially.

To bound the per agent expected regret of this system, we divide time into two parts; before $\E[A_{\hat\tau}]$ and after $\E[A_{\hat\tau}]$. The regret before time $\E[A_{\hat\tau}]$ is trivially upper bounded by $\E[A_{\hat\tau}]$ and since, after time $\E[A_{\hat\tau}]$ the set of active arms for each remains fixed, this reduces to bounding the expected regret of a classic multi-armed bandit problem. For this, we consider the approach given in \cite{lattimore2020bandit}, where we show that for a late enough time, we expect that the KL-UCB for the optimal arm does not fall far below its true mean, and additionally the KL-UCB for all suboptimal arms does not exceed this value often.


We now proceed with proof itself, which goes through a sequence of lemmas. Let's begin by introducing some notation used throughout. Firstly, fix the exploration function $f(t) := 1 + t \log ^2(t)$ (i.e. $\alpha = 1$). Next we define the suboptimality gap for each arm $k \in [K]$ by $$\Delta_k:= \mu_{\star}- \mu_k,$$
and we define the smallest suboptimality gap, 
$$\Delta_{\min} := \min_{k \in [K] \setminus \{\star\}}\Delta_k > 0.$$

For each $\epsilon \in (0,\Delta_{\min})$ and each agent $n \in [N]$ we define a random variable \begin{align*}
\kappa^n_\epsilon:= \min\left\lbrace t \in \N: \max_{s \in [T]}\left( \underline{d}\left(\hat{\mu}_{\star,s}^n,\mu_{\star}-{\epsilon}\right)-\frac{\log(f(t))}{s}\right)\leq 0\right\rbrace,
\end{align*}
where $\underline{d}(p,q):=\KL(p,q)\cdot \1\{p \leq q\}$. This random variable denotes the time where after the KL-upper confidence bound of the optimal arm will not fall below $\mu_{\star} - \epsilon$, no matter how times the optimal arm has been played.

Next we define for every $\epsilon \in (0,\Delta_{\min})$, for every agent $n \in [N]$ and for every suboptimal arm $k \in [K]\setminus \{\star\}$,
\begin{align*}
\nu_{\epsilon,k}^n:= \sum_{s=1}^T \1\left\lbrace \KL(\hat{\mu}^n_{k,s},\mu_{\star}-\epsilon)\leq \frac{\log(f(T))}{s}\right\rbrace.    
\end{align*}

This random variable denotes the number of times the KL-upper confidence bound of a suboptimal $k$ arm exceeds $\mu_{\star} - \epsilon$. 

Together, these random variables allow us to bound the regret. After time $\kappa_\epsilon^n$, the number of times any suboptimal arm is played is bounded above by $\nu_{\epsilon,k}^n$. Since these are random variables and we will consider their expected values which we will need to show are finite. Hence, we require the following two lemmas (Lemma \ref{lemma:boundingKappaEps} \& Lemma \ref{lemma:boundingNuEps}), which are essentially the same as \cite[Lemma 10.7 \& Lemma 10.8]{lattimore2020bandit}, respectively.

\begin{lemma}\label{lemma:boundingKappaEps} For ${\epsilon} \in (0,\Delta_{\min})$,  $\max_{n \in [N]}\E[\kappa^n_\epsilon] \leq 2/{\epsilon}^2$.
\end{lemma}
\begin{lemma}\label{lemma:boundingNuEps} For ${\epsilon} \in (0,\Delta_{\min})$, and $n \in [N]$ we have
\begin{align*}
\E[\nu_{\epsilon,k}^n] \leq \inf_{\tilde{\epsilon} \in (0,\Delta_k-\epsilon)} \left( \frac{\log f(T)}{\KL(\mu_k+\tilde{\epsilon},\mu_{\star}-\epsilon)}+\frac{1}{2\tilde{\epsilon}^2}\right).
\end{align*}
\end{lemma}

To continue the proof we define some further random variables that concern the optimal arm and its movement around the network.

Firstly, for each agent $n \in [N]$ and each phase $j$ we define a Boolean random variable $$\chi_j^n:=\1\{ \star \in S^n_j,~M^n_j \neq \star, A_{j-1}\geq \kappa^n_{\circ} \},$$ where $\kappa^n_{\circ}:=\kappa^n_{\Delta_{\min}/2}$. This variable indicates whether an agent has the best arm but has not played it most over the phase $j$ (and therefore it will not recommend it). Additionally, the condition $A_{j - 1} \ge \kappa_\circ^n$ demands that we are in a late enough phase which is necessary for lemma \ref{lemma:chiBound}.

For each agent $n \in [N]$ we define the following random variables:
\begin{align*}
\hat{\tau}^n_{\mathrm{stab}}&:=\min\{j \in \N~:~ A_{j-1} \geq \kappa^n_{\circ}, \forall j' \geq j,~\chi_{j'}^n=0\}\\
\hat{\tau}_{\mathrm{stab}}&:=\max_{n \in [N]}\hat{\tau}^n_{\mathrm{stab}}\\
\hat{\tau}^n_{\mathrm{spr}}&:=\min\{j \geq \hat{\tau}_{\mathrm{stab}}~:~\star \in S^n_j\}-\hat{\tau}_{\mathrm{stab}}\\
\hat{\tau}_{\mathrm{spr}}&:=\max_{n \in [N]}\hat{\tau}^n_{\mathrm{spr}}\\
\hat{\tau}&:=\hat{\tau}_{\mathrm{stab}}+\hat{\tau}_{\mathrm{spr}}.
\end{align*}
These random variables highlight two key timings of the system (for each agent). The first being the \textit{stabilisation} phase $\hat\tau_{\mathrm{stab}}^n$; this is the phase whereafter agent $n$ will always recommend the best arm if it has the best arm. The second is the \textit{spreading} time $\hat{\tau}_{\mathrm{spr}}$; this is the number of phases after $\hat\tau_{\mathrm{stab}}^n$, where agent $n$ will have the best arm for all subsequent phases. After phase $\hat\tau$, each agent will have the best arm and only recommend the best arm, therefore the set of active arms for each agent will be subsequently fixed. This is the contents of lemma \ref{lemma:stabilityLemma}.


\begin{lemma}\label{lemma:stabilityLemma} For all phases $j > \hat{\tau}$ and all $n \in [N]$ we have $S^n_j=S^n_\circ \cup \{\star\}$.
\end{lemma}
\begin{proof} For each agent $n \in [N]$, we see by induction that for any phase $j \geq \hat{\tau}^n_{\mathrm{spr}}+\hat{\tau}_{\mathrm{stab}}$, we have that $M^n_j=\star\in S^n_j$. 
    
Moreover, since $S_{j+1}^n=S_\circ^n \cup \{M^n_j, M^Q_j\}$ for some agent $Q$ in $[N]$, it follows that $S^n_{j+1}=S^n_\circ \cup \{\star\}$, for all $j \geq \hat{\tau} = \hat{\tau}_{\mathrm{stab}}+\hat{\tau}_{\mathrm{spr}}$.
\end{proof}

In the following lemma, we bound the number of times a suboptimal arm is played after the phase $\hat \tau$. 
\begin{lemma}\label{lemma:numSubOptimalPulls} For each agent $n \in [N]$ and each suboptimal arm $k \in [K] \backslash \{\star\}$ we have
\begin{align*}
\sum_{t=A_{\hat\tau}+1}^T\1\left\lbrace I^n_t = k\right\rbrace  \leq \begin{cases}
    \inf_{\epsilon \in (0,\Delta_{\min})} \left\lbrace\nu_{\epsilon,k}^n+\kappa^n_\epsilon\right\rbrace&\text{ if }k \in S^n_\circ\\
    0&\text{ if }k \notin S^n_\circ.
\end{cases}
\end{align*}
\end{lemma}
\begin{proof}
Fix an agent $n \in [N]$. First note that by Lemma \ref{lemma:stabilityLemma} we have $S^n_j= S^n_\circ \cup \{\star\}$ for all phases $j >\hat\tau$. In particular, this means that $ I^n_t \notin S^n_\circ \cup \{\star\}$ cannot occur for $t \geq A_{\hat\tau}+1$. Now take $\epsilon \in (0,\Delta_{\min})$ and consider a suboptimal arm $k \in S^n_\circ\backslash [\star]$. If $I_t^n = k$ for some $t \geq  (A_{\hat\tau}+1)\vee \kappa^n_\epsilon$ then we must have $\klucb^n_k(t-1) \geq \klucb^n_{\star}(t-1)\geq \mu_{\star}-\epsilon$, and hence,
\begin{align*}
\KL(\hat{\mu}^n_{k,{\numPlays^n_k(t-1)}},\mu_{\star}-\epsilon)\leq \frac{\log(f(t))}{\numPlays^n_k(t-1)}\leq \frac{\log(f(T))}{\numPlays^n_k(t-1)}.
\end{align*}
Consequently, 
{\small{$$\sum_{t=(A_{\hat\tau}+1)\vee \kappa^n_\epsilon}^T\1\left\lbrace I^n_t = k\right\rbrace \leq \sum_{t=(A_{\hat\tau}+1)\vee \kappa^n_\epsilon}^T\1\left\lbrace I_t^n = k \text{ and } \KL(\hat{\mu}^n_{k,{\numPlays^n_k(t-1)}},\mu_{\star}-\epsilon)\le \frac{\log(f(T))}{\numPlays^n_k(t-1)} \right\rbrace \le \nu_{\epsilon,k}^n,$$}} 
and therefore,
$$\sum_{t=A_{\hat\tau}+1}^T\1\left\lbrace I^n_t = k\right\rbrace \leq \nu_{\epsilon, k}^n + \kappa_\epsilon^n.$$
The result then follows by taking an infimum over $\epsilon \in (0,\Delta_{\min})$.
\end{proof}

This leads to the following regret bound.
\begin{cor}\label{cor:regretAfterTau} For each $n \in [N]$, we have \[\E[\mathcal{R}^n_T] \leq \E[A_{\hat{\tau}}]+\sum_{k \in S^n_\circ\backslash [\star]} \Delta_k \inf_{\epsilon \in \left(0,\frac{\Delta_{\min}}{2}\right)}\left\lbrace \frac{\log f(T)}{\KL(\mu_k+{\epsilon},\mu_{\star}-\epsilon)}+\frac{3}{{\epsilon}^2}\right\rbrace.\]
\end{cor}
\begin{proof} This follows from Lemmas \ref{lemma:boundingKappaEps}, \ref{lemma:boundingNuEps} and \ref{lemma:numSubOptimalPulls}.
\end{proof}

For the remainder of the proof we must show that $\E[A_{\hat{\tau}}]$ may be bounded independently of $T$.

We do this as follows:
In lemma \ref{lemma:chiBound}, we show that if the length of a phase is large enough, then the expected value of $\chi_j^n$ decays exponentially with phase length; In lemmas \ref{lemma:boundingSpreadingTime} and \ref{lemma:highProbBoundTauStab} with find high probability bounds for $\hat\tau_{\mathrm{stab}}$ and $\hat\tau_{\mathrm{spr}}$ respectively; And, we conclude in \ref{prop:boundATau} by showing $\E[A_{\hat \tau}]$ is finite and does not depend on time horizon $T$.

\begin{lemma}\label{lemma:chiBound} For every phase $j \in \N$ such that $A_j-A_{j-1} \geq \frac{8}{\Delta^2}\left( \frac{K}{N}+3\right) {\log f(A_j)}$, we have \[\E[\chi^n_j]\leq \frac{8K}{\Delta_{\min}^2} \exp\left( -\frac{\Delta_{\min}^2 (A_j-A_{j-1})}{16 (K/N+3)}\right).\]
\end{lemma}
\begin{proof} First observe that if $\chi_j^n=1$ then $\star \in S^n_j$, $A_{j-1}\geq \kappa^n_{\circ}$ and $M^n_j \neq \star$. Since $M^n_j \neq \star$ we deduce that for some $k \in [K]\backslash \{\star\}$, we have 
\[\numPlays_k^n(A_j)-\numPlays_k^n(A_{j-1}) \geq \frac{A_j-A_{j-1}}{|S^n_j|} \geq \frac{A_j-A_{j-1}}{K/N+3},\]
and so for some $A_{j-1}<t\leq A_j$ we have $s=\numPlays_k^n(t-1) \geq \frac{A_j-A_{j-1}}{K/N+3}-1$ and $I_t^n=k$, so $\klucb^n_{k}(t-1) \geq \klucb^n_{\star}(t-1)$ as $\star \in S^n_j$. Since $t\geq A_{j-1} \geq \kappa^n_{\circ}$ we deduce that $\klucb^n_{k}(t-1) \geq \klucb^n_{\star}(t-1) \geq \mu_{\star}-\Delta_{\min}/2$. Hence,  by Pinsker's inequality 
\begin{align*}
2 \left(\hat\mu^{n}_{k,s}- \mu_\star+\frac{\Delta_{ \min }}{2}\right)^2=2 \left(\hat\mu^{n}_{k}(t - 1)- \mu_\star+\frac{\Delta_{ \min }}{2}\right)^2&\leq \KL\left(\hat\mu^{n}_{k}(t - 1), \mu_{\star}-\frac{\Delta_{ \min }}{2}\right) \\
&\leq \frac{\log(f_\alpha(t))}{\numPlays_k^n(t - 1)} \leq \frac{\log f(A_j)}{s}.    
\end{align*}
Thus, for some $k \in [K]\backslash \{\star\}$ and $s \geq \frac{A_j-A_{j-1}}{K/N+3}-1$,
\begin{align*}
\hat\mu^{n}_{k,s} &\geq \mu_{\star}-\frac{\Delta_{\min}}{2}-\sqrt\frac{\log f(A_j)}{2s}\geq \mu_k+\frac{\Delta_{\min}}{2}-\sqrt\frac{\log f(A_j)}{2s}\geq \mu_k+\frac{\Delta_{\min}}{4},
\end{align*}
since $A_j-A_{j-1} \geq \frac{8}{\Delta^2}\left( \frac{K}{N}+3\right) {\log f(A_j)}$. Thus, by Hoeffding's inequality we have
\begin{align*}
\E[\chi^n_j]& \leq \sum_{k \in [K]\backslash\{\star\}}\sum_{s \geq \frac{A_j-A_{j-1}}{K/N+3}-1}\Prob\left[\hat\mu^{n}_{k,s} \geq  \mu_k+\frac{\Delta_{\min}}{4}\right]\\
& \leq (K-1)\sum_{s \geq \frac{A_j-A_{j-1}}{K/N+3}-1} \exp\left(-\frac{s \Delta_{\min}^2}{8}\right)\\
&\leq K\int_{\frac{A_j-A_{j-1}}{K/N+3}-2}^{\infty} \exp\left(-\frac{s \Delta_{\min}^2}{8}\right)ds\\
&\leq \frac{8K}{\Delta_{\min}^2} \exp\left( -\frac{\Delta_{\min}^2 (A_j-A_{j-1})}{16 (K/N+3)}\right).
\end{align*}
\end{proof}

In what follows we let $p_{\min}:=\min\left(\{P(i,j)\}_{(i,j) \in [N]^2}\backslash \{0\}\right)$ and $\diam(P)$ denote the maximum length of a directed path between two distinct nodes corresponding to the graph induced by $P$. We note that $\diam(P)<\infty$ if and only if $P$ has a strongly connected graph.

\begin{lemma}\label{lemma:boundingSpreadingTime} Suppose that $P$ has a strongly connected graph. Then for $\xi \in \N$, $\Prob(\hat{\tau}_{\mathrm{spr}}\geq \xi) \leq  N(1-p_{\min}^{\diam(P)})^{\left\lfloor\frac{\xi}{2\diam(P)}-1\right\rfloor}$.
\end{lemma}
\begin{proof} 
    Recall that $\hat\tau_{\mathrm{spr}}$ is the number of phases \textit{since} $\hat\tau_{\mathrm{stab}}$, so we can assume that if an agent has the best arm it will recommend it. Therefore, to find an upper bound for this probability, we consider a single path from an agent with the optimal arm ($n_{\star}$) to the chosen node $n$ and the probability that there exists a single node in this path does not request a recommendation from the prior node. And therefore, the best arm does not spread along this path. 
    
    Fix an agent $n \in [N]$ and choose a sequence of nodes $(\ell_i)_{i \in [q]\cup\{0\}} \in [N]^q$ with $q \leq \diam(P)$ and such that $\ell_0=n_{\star}$, $\ell_q=n$ and $P(\ell_{i},\ell_{i-1})>0$ for each $i \in [q]$. Note that the definition of $\diam(P)$ entails the existence of at least one such a sequence. Recall that we let $Q_j^{\tilde{n}}$ denote the node which sends a message to agent $\tilde{n}$ and the end of phase $j$. Let $m=\lfloor \xi/(2q)-1\rfloor$ and observe that if for some $j_0\in \{\hat{\tau}_{\mathrm{stab}}, \ldots, \hat{\tau}_{\mathrm{stab}}+2mq\} $ we have $Q_j^{\ell_{j-j_0}}=\ell_{j-j_0-1}$ for $j \in \{j_0+1,\ldots,j_0+q\}$ then $\hat{\tau}^n_{\mathrm{spr}}+\hat{\tau}_{\mathrm{stab}}\leq j_0+q  < \xi+\hat{\tau}_{\mathrm{stab}}$. Hence, we have
\begin{align*}
\Prob(\hat{\tau}_{\mathrm{spr}}^n \geq \xi) & \leq \Prob\left( \bigcap_{j_0 -\hat{\tau}_{\mathrm{stab}} \in \{0,1, \ldots,  2mq\}} \bigcup_{j \in \{j_0+1,\ldots,j_0+q\}} \left\lbrace Q_j^{\ell_{j-j_0}}\neq \ell_{j-j_0-1} \right\rbrace \right)\\
& \leq \Prob\left( \bigcap_{j_0 -\hat{\tau}_{\mathrm{stab}} \in \{0,2q, \ldots,  2mq\}} \bigcup_{j \in \{j_0+1,\ldots,j_0+q\}} \left\lbrace Q_j^{\ell_{j-j_0}}\neq \ell_{j-j_0-1} \right\rbrace \right)\\
& = \prod_{j_0 -\hat{\tau}_{\mathrm{stab}} \in \{0,2q, \ldots,  2mq\}} \Prob\left(  \bigcup_{j \in \{j_0+1,\ldots,j_0+q\}} \left\lbrace Q_j^{\ell_{j-j_0}}\neq \ell_{j-j_0-1} \right\rbrace \right)\\
& = \prod_{j_0 -\hat{\tau}_{\mathrm{stab}} \in \{0,2q, \ldots,  2mq\}}  \left\lbrace1-\Prob\left(  \bigcap_{j \in \{j_0+1,\ldots,j_0+q\}} \left\lbrace Q_j^{\ell_{j-j_0}} = \ell_{j-j_0-1} \right\rbrace \right)\right\rbrace\\
& =  \prod_{j_0 -\hat{\tau}_{\mathrm{stab}} \in \{0,2q, \ldots,  2mq\}}  \left\lbrace 1-\prod_{j \in \{j_0+1,\ldots,j_0+q\}} \Prob\left(   Q_j^{\ell_{j-j_0}} = \ell_{j-j_0-1}  \right)\right\rbrace \\
&\leq (1-p_{\min}^q)^m \leq (1-p_{\min}^{\diam(P)})^{\left\lfloor\frac{\xi}{2\diam(P)}-1\right\rfloor}.
\end{align*}
The lemma now follows by the union bound over $[N]$.
\end{proof}

The following lemma gives us a bound for the time at which each phase starts (and ends) by considering the phase lengths. This has a simple inductive proof which we omit. 
\begin{lemma}\label{lemma:ABoundsUL} Suppose that there exist $C\geq 1$, $\theta>0$ such that $C^{-1} j^\theta \leq A_{j}-A_{j-1} \leq Cj^{\theta}$ for all $j \in \N$. Then we have $C^{-1} j^{1+\theta} \leq A_j \leq C(1+j)^{1+\theta}$ for all $j \in \N$.
\end{lemma}

Now define the phase $\underline{j}(\Delta_{\min}) \in \N$ by
\begin{align*}
\underline{j}(\Delta_{\min}):=1+\max\left(\{0\}\cup\left\lbrace j \in \N~:~ j^{\theta} < \frac{8C}{\Delta_{\min}^2}\left(\frac{K}{N}+3\right)\log f\left( C(1+j)^{\theta}\right) \right\rbrace\right).
\end{align*}
Note that $\underline{j}(\Delta_{\min})$ is always finite since $f(t)=O(\log t)$. This phase is conveniently defined by considering lemma \ref{lemma:ABoundsUL} and with the purpose of applying lemma \ref{lemma:chiBound} in lemma 
\ref{lemma:highProbBoundTauStab}.

\begin{lemma}\label{lemma:highProbBoundTauStab} Suppose that there exist constants $C\geq 1$, $\theta>0$ such that $C^{-1} j^\theta \leq A_{j}-A_{j-1} \leq Cj^{\theta}$ for all $j \in \N$. Then for all $\xi \geq \underline{j}(\Delta_{\min})$ we have 
\begin{align*}
\Prob( \hat{\tau}_{\mathrm{stab}} \geq \xi) \leq \sum_{n \in [N]}\Prob(\kappa^n_{\circ}> C^{-1}(\xi-2)^{1+\theta})+\frac{8KN}{\Delta_{\min}^2} \sum_{j \geq \xi}\exp\left(-\frac{\Delta_{\min}^2j^\theta}{16C(K/N+3)}\right).
\end{align*}
\end{lemma}
\begin{proof} Fix an agent $n \in [N]$ and suppose that $\hat{\tau}^n_{\mathrm{stab}} \geq \xi$. Since $\hat{\tau}^n_{\mathrm{stab}}:=\min\{j \in \N~:~ A_{j-1} \geq   \kappa^n_{\circ}, \forall j' \geq j,~\chi_{j'}^n=0\}$ it follows that either $A_{\xi-2} < \kappa^n_{\circ}$ or $\chi_{j}^n=1$ for some $j \geq \xi$. Note also that by the upper bound in Lemma \ref{lemma:ABoundsUL}  for $j \geq \xi \geq \underline{j}(\Delta_{\min})$ we have
\begin{align*}
A_j-A_{j-1} \geq C^{-1} j^\theta \geq \frac{8}{\Delta_{\min}^2}\left(\frac{K}{N}+3\right)\log f\left( C(1+j)^{\theta}\right) \geq \frac{8}{\Delta^2}\left( \frac{K}{N}+3\right) {\log f(A_j)}.
\end{align*}
Hence, by Lemmas \ref{lemma:chiBound} and the lower bound in \ref{lemma:ABoundsUL} we have
\begin{align*}
\Prob( \hat{\tau}^n_{\mathrm{stab}} \geq \xi) &\leq \Prob( A_{\xi-2} < \kappa^n_{\circ})+\sum_{j \geq \xi}\E[\chi^n_j]\\
& \leq \Prob(\kappa^n_{\circ}> C^{-1}(\xi-2)^{1+\theta})+\frac{8K}{\Delta_{\min}^2} \sum_{j \geq \xi}\exp\left( -\frac{\Delta_{\min}^2 (A_j-A_{j-1})}{16 (K/N+3)}\right)\\
& \leq \Prob(\kappa^n_{\circ}> C^{-1}(\xi-2)^{1+\theta})+\frac{8K}{\Delta_{\min}^2} \sum_{j \geq \xi}\exp\left(-\frac{\Delta_{\min}^2j^\theta}{16C(K/N+3)}\right)\\
& \leq \Prob(\kappa^n_{\circ}> C^{-1}(\xi-2)^{1+\theta})+\frac{8K}{\Delta_{\min}^2} \int_{z \geq \xi-1}\exp\left(-\frac{\Delta_{\min}^2z^\theta}{16C(K/N+3)}\right)dz.
\end{align*}
Once again conclusion of the lemma follows by union bounding over $n \in [N]$.
\end{proof}

\begin{prop}\label{prop:boundATau} Suppose that there exist $C\geq 1$, $\theta>0$ such that $C^{-1} j^\theta \leq A_{j}-A_{j-1} \leq Cj^{\theta}$ for all $j \in \N$. Then there exists a constant $\phi \equiv \phi(\Delta_{\min},C,\theta,N,K,p_{\min},\diam(P))$ depending on $\Delta_{\min},C,\theta,N,K,p_{\min},\diam(P)$ but not $T$ such that $\E[A_{\tau}] \leq \phi$.
\end{prop}
\begin{proof} Given  $A_{\hat\tau} \geq \zeta \geq C(1+2\underline{j}(\Delta_{\min}))^{1+\theta} \vee C \cdot \{16 \diam(P)\}^{1+\theta}$ then  $\hat{\tau} \geq (\zeta/C)^{\frac{1}{1+\theta}}-1$, so  $\hat{\tau}_{\mathrm{spr}} \vee \hat{\tau}_{\mathrm{stab}} \geq \{(\zeta/C)^{\frac{1}{1+\theta}}-1\}/2\geq \underline{j}(\Delta_{\min})$. Hence, for $\zeta \geq \psi \equiv \psi(\Delta_{\min},C,\theta):= C(1+2\underline{j}(\Delta_{\min}))^{1+\theta}\vee C \{16 \diam(P)\}^{1+\theta}$,
\begin{align*}
\Prob(A_{\hat\tau} \geq \zeta) & \leq \Prob\left( \hat{\tau}_{\mathrm{spr}} \geq \frac{1}{2}\{(\zeta/C)^{\frac{1}{1+\theta}}-1\}\right)+\Prob\left(  \hat{\tau}_{\mathrm{stab}} \geq \frac{1}{2}\{(\zeta/C)^{\frac{1}{1+\theta}}-1\}\right)\\
& \leq N(1-p_{\min}^{\diam(P)})^{\big\lfloor\frac{(\zeta/C)^{\frac{1}{1+\theta}}}{4\diam(P)}-2\big\rfloor}+\frac{8KN}{\Delta_{\min}^2} \int_{z \geq (\zeta/C)^{\frac{1}{1+\theta}}/2-2}\exp\left(-\frac{\Delta_{\min}^2z^\theta}{16C(K/N+3)}\right)dz\\&\hspace{1cm}+\sum_{n \in [N]}\Prob(\kappa^n_{\circ}> \{(\zeta/C)^{\frac{1}{1+\theta}}/2-4\}^{1+\theta}/C)\\
& \leq N(1-p_{\min}^{\diam(P)})^{\frac{(\zeta/C)^{\frac{1}{1+\theta}}}{2^4\diam(P)}}+\frac{8KN}{\Delta_{\min}^2} \int_{z \geq (\zeta/C)^{\frac{1}{1+\theta}}/2-2}\exp\left(-\frac{\Delta_{\min}^2z^\theta}{16C(K/N+3)}\right)dz\\&\hspace{1cm}+\sum_{n \in N}\Prob(\kappa^n_{\circ}> (2^{1+\theta}C)^{-2} \cdot \zeta).
\end{align*}
Note also that by Lemma \ref{lemma:boundingKappaEps} we have
\begin{align*}
\sum_{n \in N}\sum_{\zeta \in \N}\Prob(\kappa^n_{\circ}> (2^{1+\theta}C)^{-2} \cdot \zeta) =  \sum_{n \in N}\sum_{\zeta \in \N }\Prob((2^{1+\theta}C)^{2}\kappa^n_{\circ}>  \cdot \zeta) =  (2^{1+\theta}C)^{2}\sum_{n \in N} \E[k^n_\circ] \leq \frac{8N}{\Delta_{\min}^2}.   
\end{align*}
Hence, we have
\begin{align*}
\E[A_{\hat\tau}] & \leq \psi + \sum_{\zeta >  \psi }\Prob(A_{\hat\tau} \geq \zeta)\\
& \leq \psi + \sum_{\zeta >  \psi } \left\lbrace  N(1-p_{\min}^{\diam(P)})^{\frac{(\zeta/C)^{\frac{1}{1+\theta}}}{2^4\diam(P)}}+\frac{8K}{\Delta_{\min}^2} \int_{z \geq (\zeta/C)^{\frac{1}{1+\theta}}/2-2}\exp\left(-\frac{\Delta_{\min}^2z^\theta}{16C(K/N+3)}\right)dz\right\rbrace\\&\hspace{1cm}+\sum_{n \in [N]}\sum_{\zeta \geq \psi }\Prob(\kappa^n_{\circ}> (2^{1+\theta}C)^{-2} \cdot \zeta)=:\phi(\Delta_{\min},C,\theta,N,K,p_{\min},\diam(P))<\infty,
\end{align*}
for a finite constant $\phi \equiv \phi(\Delta_{\min},C,\theta,N,K,p_{\min},\diam(P))$. which depends on $\Delta_{\min}$, $C$, $\theta$, $N$, $K$, $p_{\min}$ and $\diam(P)$, but not on $T$.
\end{proof}

We can now complete the proof of the main result. 

\begin{proof}[Proof of Theorem \ref{thm:asymptoticRegretBoun}] The result follows from Corollary \ref{cor:regretAfterTau} combined with Proposition \ref{prop:boundATau} by taking $\epsilon \rightarrow 0$.
\end{proof}

\section{Numerical Results}\label{sec:numericalResults}
Here we will compare algorithm~\ref{alg:klalg} and the GosInE algorithm on a range of synthetic data. We trial variants of both of these algorithms using Hoeffding and KL upper confidence bounds. For GosInE, the Hoeffding and KL variants are respectively labelled UCB-GIE and KLUCB-GIE and for algorithm~\ref{alg:klalg}, they are labelled GIE-FE (Gossip-Insert-Eliminate with Fast Elimination) and AOGB.

All the experiments are conducted in two settings $N, K = (20, 50)$ and $N, K = (10, 100)$ where in both settings, the phases are growing cubically, i.e., $A_j = j^3$. Each experiment consists of 100 independent runs, and in each run the regret is averaged over the nodes. In each experiment, the algorithms encounter the same reward sequence. The first two experiments assume the agents are connected via a complete graph, while the third experiment compares different graphs. We compute the regret over a time horizon of $T = 100,000$ and plot the mean along with 95\% confidence intervals.

\noindent\textbf{Choice of $\alpha$: }
We begin by comparing algorithm~\ref{alg:klalg} and GosInE for the two different types of upper confidence bounds by varying the exploration function $f(t) = 1 + t^\alpha \log^2(t)$ by choosing different values for $\alpha$.
\begin{figure}[htp]
\centering
\subfloat[$N = 20, K = 50$]{\includegraphics[width = 0.45\textwidth]{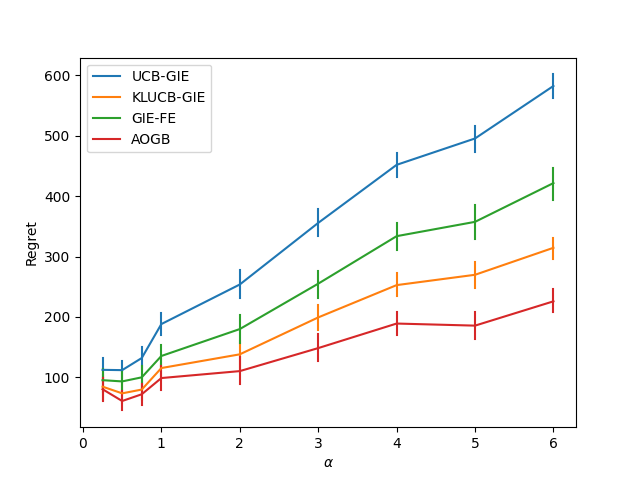}}%
\qquad
\subfloat[$N = 10, K = 100$]{\includegraphics[width = 0.45\textwidth]{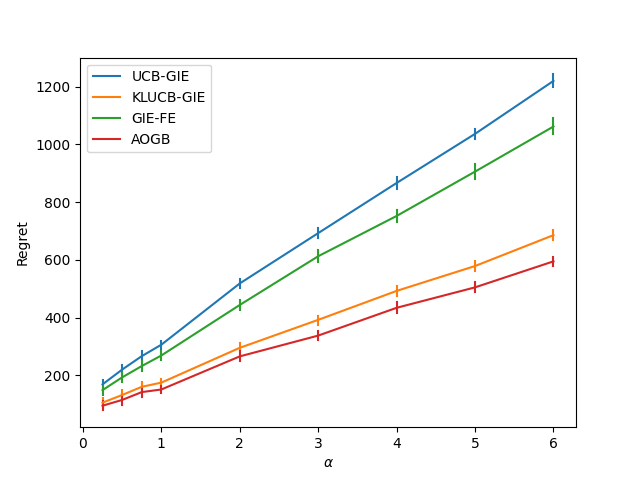}}%
\caption{Regret for different choices of $\alpha$ with $\mu_{\star} = 0.9$ and the rest of the arms divide the interval $[0. 2, 0.8]$ uniformly.}%
\label{fig:alpha}%
\end{figure}
From figure \ref{fig:alpha} we identify that algorithm~\ref{alg:klalg} and GosInE perform better when equipped with KL upper confidence bound. Additionally, algorithm~\ref{alg:klalg} outperforms GosInE when they are both equipped with the same upper confidence bounds. Overall, performance is better for the smaller values of $\alpha$ and regret is minimised somewhere in the region $\alpha \le 1$. This implies that there may be more practical choices for $f_\alpha(t)$ than the asymptotically optimal choice at $\alpha = 1$.

\noindent\textbf{$\Delta_{\min}$ vs Regret:} Now we consider the affect of changing the sub-optimality gap $\Delta_{\min}$. This is the difference between the mean of the best arm  and the second best arm. Figure \ref{fig:delta2} compares algorithm~\ref{alg:klalg} and the GosInE algorithm for both types of confidence intervals. 
\begin{figure}[htp]
\centering
\subfloat[$N = 20, K = 50$]{\includegraphics[width = 0.45\textwidth]{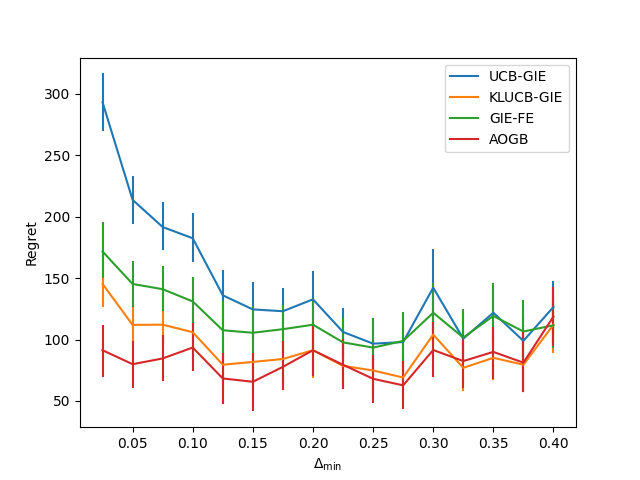}}
\subfloat[$N = 10, K = 100$]{\includegraphics[width = 0.45\textwidth]{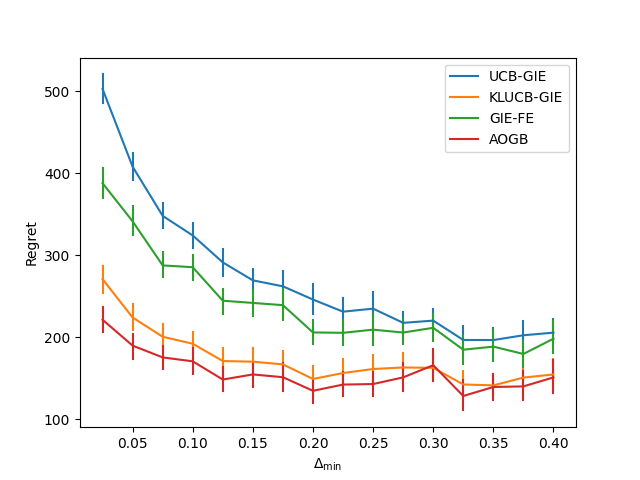}}
\caption{Regret for different choices of $\Delta_{\min}$  with $\alpha = 1$. The best arm has mean $\mu_\star = 0.9$ and the rest of the arms divide the interval $[0.9 - \Delta_{\min}, 0.2]$ uniformly.}
\label{fig:delta2}
\end{figure}
Similarly to the previous experiment, we observe that both algorithms perform better when equipped with the KL upper confidence bounds and that algorithm~\ref{alg:klalg} typically outperforms GosInE on average when they are equipped with the same upper confidence bounds.  

\noindent\textbf{Network Configurations: }
Here, we compare three different network configurations for agents  implementing algorithm~\ref{alg:klalg}: a complete graph, a cycle graph and a star graph.

\begin{figure}[htp]
\centering
\subfloat[$N = 20, K = 50$]{\includegraphics[width = 0.45\textwidth]{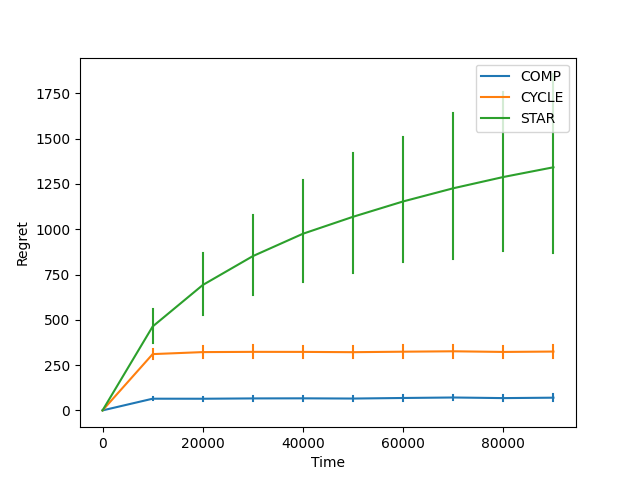}}
\subfloat[$N = 10, K = 100$]{\includegraphics[width = 0.45\textwidth]{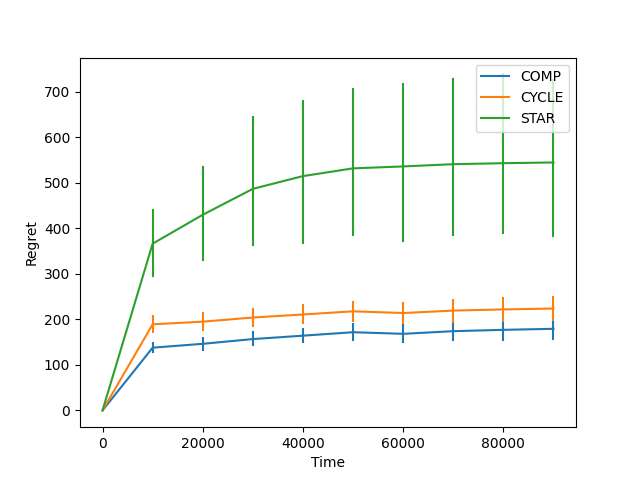}}
\caption{Regret over time for three different networks. Each in case we consider $\alpha = 1$, $\Delta_{\min} = 0.1$ and the means of the remaining arms divide the interval $[0.8, 0.2]$ uniformly.}
\label{fig:connectivity}
\end{figure}

The results in figure~\ref{fig:connectivity} show that the cycle graph performs slightly worse than the complete graph but the star graph struggles significantly along with a larger variance. In essence, this is because the best arm needs to spread to centre of the star before it can spread to all of the other nodes.

\section{Discussion}

In this paper we presented an algorithm (Algorithm \ref{alg:klalg}) for multi-agent bandits in a decentralised setting. Our algorithm builds upon the Gossip-Insert-Eliminate algorithm of \cite{chawla2020gossiping} by making two modifications. First, we use tighter confidence intervals inspired by \cite{garivier2011kl}. Second, we use a faster elimination scheme for reducing the number of arms that must be explored by an agent. Both modifications yield significant empirical improvement on simulated data (Figure \ref{fig:delta2}). Finally, we prove a regret bound (Theorem \ref{thm:asymptoticRegretBoun}) which demonstrates asymptotically optimal performance of our algorithm, matching the asymptotic performance of a collection of agents with unlimited communication.

There is substantial scope for future work in this direction. One challenge of great practical importance is the development of distributed algorithms which are robust to both malicious agents and faulty communication \cite{lynch1996distributed}. An interesting theoretical challenge is to develop a multi-agent bandit algorithm which is both asymptotically optimal and nearly minimax optimal with limited communication. In very recent work of \cite{agarwal2021multi} an algorithm has been proposed which is minimax optimal in the distributed setting, and it would be interesting to synthesise this with the insights provided in the current paper.

\bibliographystyle{plain}
\bibliography{bib}
\end{document}